\documentclass[10pt,twocolumn]{article}
\usepackage[margin=0.75in]{geometry}
\usepackage{amsmath,amssymb,graphicx,booktabs}
\usepackage[hidelinks]{hyperref}
\usepackage{titlesec}
\titleformat{\section}{\normalfont\large\bfseries}{\thesection}{0.6em}{}
\titleformat{\subsection}{\normalfont\normalsize\bfseries}{\thesubsection}{0.5em}{}
\newtheorem{proposition}{Proposition}
\newtheorem{corollary}{Corollary}

\title{\vspace{-2em}\textbf{The Sparsity Ceiling: Where Spiking Networks Can---\\and Cannot---Trade Activity for Energy}}
\author{Zeyu Wang\\ \small Georgia Institute of Technology \\ \small \texttt{zwang3286@gatech.edu}}
\date{\today}

\begin{document}
\maketitle

\begin{abstract}
\noindent Spiking neural networks (SNNs) are promoted as an energy-efficient substrate because sparse,
event-driven activity replaces dense multiply--accumulates (MACs) with cheap accumulates (ACs). We argue
that the energy dividend of sparsity is \emph{not a property of SNNs but of the task}. Using a controlled
protocol---architecture held fixed, only the hidden unit swapped (continuous vs.\ leaky-integrate-and-fire),
plus a two-sided target-firing-rate probe---we measure how far activity can be pushed down before quality
breaks. On feed-forward perception, hidden firing is naturally low and drives precisely to a $5\%$ target with
no accuracy loss. On recurrent character-level language modeling the same probe, targeting $10\%$, cannot move
firing below $\sim\!50\%$: the recurrent state must stay active to carry information. A spiking
\emph{Transformer}, by contrast, sparsifies freely to $2\%$ firing at no extra quality cost (3 seeds) --- so
the ceiling is a property of \emph{recurrent compression}, not of sequence modeling. Attention escapes the floor
only by storing the full key--value cache, trading a firing floor for a memory wall: neither architecture
escapes cost on neuromorphic hardware, they pay on different axes. We formalize the ceiling with an
information-theoretic firing-floor bound
$\rho \ge H_b^{-1}(\log_2 M / H)$ and confirm its predictions on synthetic tasks:
the floor rises with memory load and falls with state width, and---refuting a naive memory-only reading---also
rises with task difficulty, exactly as a load-based bound requires. A layer-wise accounting exposes an \emph{input floor} that further caps op
reduction under dense (frame-replayed) input, isolating native event-driven input as the setting where
neuromorphic hardware wins. The results recast ``SNNs save energy'' as a task-conditional claim with a concrete,
measurable rule for where neuromorphic computation pays off.
\end{abstract}

\section{Introduction}
The standard case for neuromorphic computation is energetic: spikes are binary and sparse, so inference reduces
to accumulate (AC) operations gated by activity, avoiding the dense multiply--accumulates (MAC) of conventional
networks. This motivates a large literature reporting $10$--$100\times$ energy reductions and, increasingly,
spiking \emph{language} models~\cite{spikegpt,spikingbrain,sdt}.

We ask a narrower, falsifiable question: \emph{when a network is otherwise identical, how much energy does
spiking actually buy, and what determines that number?} Prior work overwhelmingly optimizes accuracy on a
single task and reports an energy estimate at whatever firing rate results; it rarely isolates the causal
variable---attainable sparsity---or asks whether that variable is controllable in a given task. Our
contributions:
\begin{itemize}\itemsep2pt
\item \textbf{Protocol.} A matched-architecture comparison changing only the neuron model, plus a two-sided
target-rate regularizer that \emph{probes} attainable sparsity rather than accepting the incidental one.
\item \textbf{The sparsity ceiling.} Low-load feed-forward perception is freely sparsifiable (to $\sim\!5\%$
firing); recurrent sequence modeling is not (firing pinned near $50$--$63\%$). Sparsity availability is
task-structural.
\item \textbf{Theory + consequence.} An information-theoretic firing-floor bound $\rho\ge H_b^{-1}(\log_2 M/H)$
explains the ceiling; controlled sweeps confirm the floor rises with memory load and falls with state width, and
(refuting a naive memory-only reading) also rises with task difficulty---the floor tracks representational load
$\log_2 M$, which memory blows up exponentially. A layer-wise \emph{input floor} shows dense-input replay caps
op reduction.
\item \textbf{Recurrence vs.\ attention: a cost dichotomy.} A spiking Transformer sparsifies freely to $2\%$
(no floor), localizing the ceiling to recurrent compression; attention pays instead in key--value memory. The
practical rule: neuromorphic wins on \emph{event-driven perception}; on sequence models it faces a firing floor
(recurrent) or a memory wall (attention).
\end{itemize}

\section{Method}
\textbf{Matched-architecture protocol.} For each task we instantiate one architecture and two variants
differing \emph{only} in the hidden nonlinearity: a conventional unit (ReLU for the convolutional perception
net; tanh for the recurrent LM) and a leaky-integrate-and-fire (LIF) spiking unit with surrogate-gradient
training. All weight shapes, depths, embeddings and readouts are identical, so any energy gap is attributable
to spiking alone.

\textbf{Attainable-sparsity probe.} We add a two-sided target-rate regularizer
$\mathcal{L}_{\mathrm{reg}}=\sum_{\ell}\lvert \bar{s}_\ell-\rho^\star\rvert$ with warmup, penalizing deviation
from a target rate $\rho^\star$ in \emph{both} directions. Unlike one-sided $L_1$ (which collapses the network
into a silent, unrecoverable state), this lets us \emph{request} a firing rate and observe whether the task
permits it.

\textbf{Energy proxy.} Per 45\,nm estimates~\cite{horowitz}, MAC $=4.6$\,pJ, AC $=0.9$\,pJ. Dense transforms
fed by continuous signals count as MACs; spike-driven transforms count as ACs scaled by measured presynaptic
firing and timesteps $T$. We separate the \emph{sparsity-driven} component from the fixed \emph{per-op}
(AC$<$MAC) component.

\section{The sparsity ceiling: a firing-floor bound}
Consider a spiking recurrent network with $H$ hidden units communicating by binary spikes
$s_t\in\{0,1\}^H$, whose recurrent and readout pathways depend on the state \emph{only through} $s_t$---the
operative regime on spike-based neuromorphic hardware. Let $\rho=\mathbb{E}[\overline{s_t}]$.

\begin{proposition}[Firing floor from memory load]
If the task requires distinguishing among $M$ memory contents per recall step (copying $N$ symbols from an
alphabet of size $S$ gives $M=S^N$), then, taking the branch $\rho\le 1/2$,
\[
\rho \;\ge\; H_b^{-1}\!\left(\tfrac{\log_2 M}{H}\right),\qquad
H_b(\rho)=-\rho\log_2\rho-(1-\rho)\log_2(1-\rho).
\]
Hence $\rho_{\min}$ is strictly increasing in the memory load $\log_2 M=N\log_2 S$, and $\rho_{\min}\to 0$ only
as $\log_2 M/H\to 0$.
\end{proposition}
\noindent\emph{Proof sketch.} The spike vector $s_t$ is the sole information channel out of the state. A binary
vector of length $H$ with at most $k=\rho H$ ones indexes at most $\sum_{i\le k}\binom{H}{i}\le 2^{H H_b(\rho)}$
messages ($\rho\le1/2$). Separating $M$ contents requires $2^{HH_b(\rho)}\ge M$, i.e. $HH_b(\rho)\ge\log_2 M$. \hfill$\square$

\begin{corollary}[Load, not memory, sets the floor]
If the task emits one decision among $C$ classes (not maintained across time), the constraint is
$HH_b(\rho)\ge\log_2 C$ at the readout \emph{only}, so $\rho$ can approach a far smaller minimum when $C\ll M$.
Memory is thus not privileged \emph{per se}: any source of representational load raises the floor. What makes
sequence modeling special is that memory makes $M{=}S^N$ grow \emph{exponentially} in the horizon, whereas
typical (low-class-count) perception carries a modest $M$ and sparsifies freely.
\end{corollary}
\S\ref{fig:dissoc} confirms both halves: a memoryless task's floor still rises with $C$, yet the exponential
memory term is what floors recurrent language modeling. \emph{Scope:} the bound assumes spike-only state communication; reading
the continuous membrane potential loosens it, but forfeits the spike-based communication that defines
neuromorphic hardware.

\section{Experiments}
Results are mean\,$\pm$\,std over 3 seeds unless noted.

\begin{figure}[t]
\centering
\includegraphics[width=\linewidth]{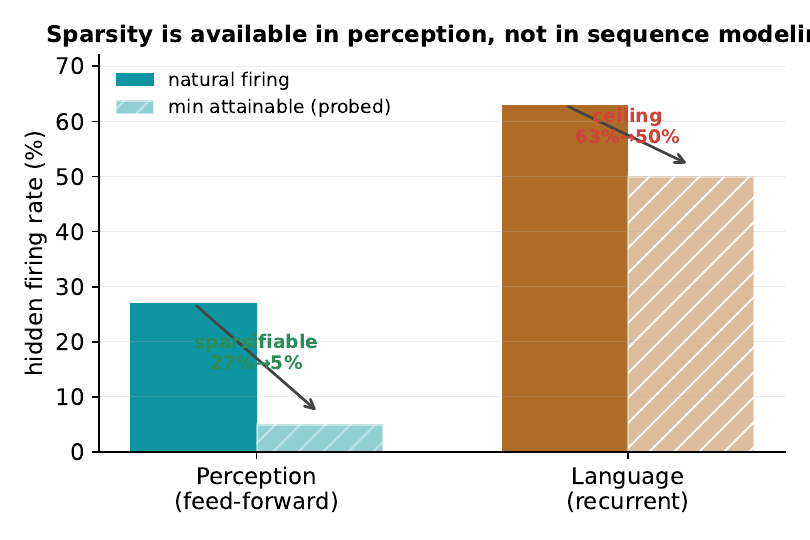}
\caption{The paper in one plot. The two-sided probe drives perception firing from $27\%$ to $5\%$ with no
accuracy loss, but cannot push recurrent language below $\sim\!50\%$: sparsity is available in feed-forward
perception and structurally unavailable in recurrent sequence modeling.}
\label{fig:contrast}
\end{figure}

\subsection{Perception is freely sparsifiable}
Convolutional SNN vs.\ matched CNN (rate-coded FashionMNIST as an event-stream stand-in; native event data is
future work). Firing drives precisely to the $5\%$ target with accuracy on par with the unregularized SNN;
one-sided $L_1$ collapses the network.
\begin{center}\small
\begin{tabular}{@{}lcccc@{}}
\toprule
config & firing & acc & ops/ANN & energy/ANN\\
\midrule
ANN (CNN) & --- & 88.1\% & $1.0\times$ & $1.0\times$\\
SNN, no reg & .16/.27 & 74.7\% & $0.52\times$ & $2.7\times$\\
SNN, $L_1$ & 0/0 & 10\% & --- & collapse\\
\textbf{SNN, $\rho^\star{=}.05$} & \textbf{.050/.051} & \textbf{76.6\,$\pm$.4} & $0.99\times$ & \textbf{5.05$\times$}\\
\bottomrule
\end{tabular}
\end{center}

\subsection{Sequence modeling hits a ceiling}
Spiking recurrent LM vs.\ architecturally identical tanh RNN (char-level WikiText-103). A strong penalty
targeting $10\%$ moves firing only to $0.50$ (from $0.63$), tight across seeds; the quality gap exceeds
perception's, and the SNN trains $\sim\!6\times$ slower on GPU.
\begin{center}\small
\begin{tabular}{@{}lccc@{}}
\toprule
config & firing & bpc & energy/ANN\\
\midrule
ANN (tanh RNN) & --- & 2.62 & $1.0\times$\\
SNN, no reg & $0.63\pm.03$ & $3.45\pm.03$ & $4.6\times$\\
\textbf{SNN, $\rho^\star{=}.10$} & $\mathbf{0.50\pm.01}$ & $3.48\pm.02$ & $5.1\times$\\
\bottomrule
\end{tabular}
\end{center}

\subsection{The firing floor rises with memory load}
On a synthetic copy task (ANN solves all, acc $1.00$), we trace accuracy vs.\ attained firing for memory loads
$N\in\{4,6,8\}$ (Fig.~\ref{fig:mech}). At $N{=}4$ the task is solved at the lowest probed firing ($\sim\!4\%$);
at $N{=}8$ accuracy collapses below $\sim\!8\%$ firing and only recovers there. The minimum firing that sustains
accuracy \textbf{rises} ($\sim\!4\%\!\to\!\sim\!8\%$), corroborating the bound. Accuracy is also
\emph{non-monotone} in firing---excess spikes hurt---indicating an optimal mid-low band.
\begin{figure}[t]
\centering
\includegraphics[width=\linewidth]{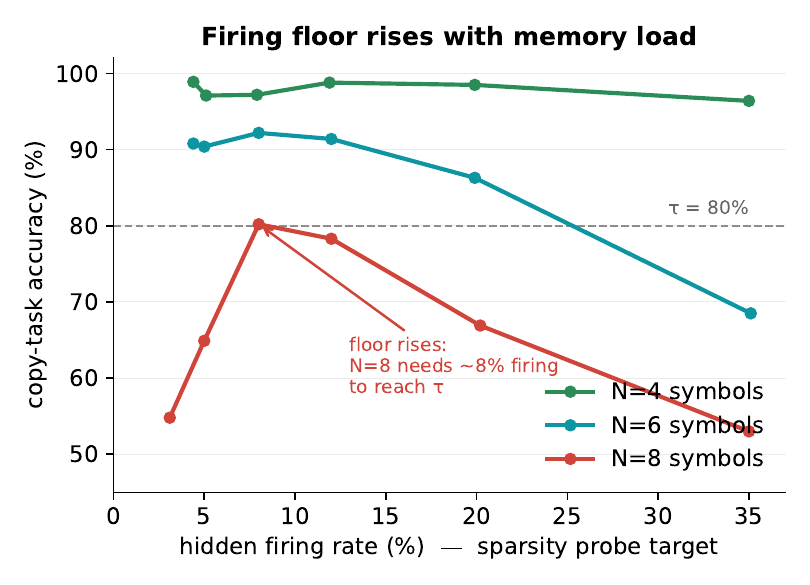}
\caption{Copy-task accuracy vs.\ hidden firing rate for three memory loads $N$. The firing needed to reach the
accuracy threshold rises with $N$: the sparsity floor tracks memory demand, as predicted by Prop.~1.}
\label{fig:mech}
\end{figure}

\subsection{The input floor}
Layer-wise op accounting at $\rho^\star{=}.05$ on perception: even with hidden layers at $5\%$ firing, the input
layer---a dense frame replayed over $T{=}10$ steps---is $61\%$ of the op budget, capping op reduction at parity
($0.99\times$). This isolates \emph{native event-driven input} (sparse by construction) as the missing
ingredient for a true op reduction, not merely an energy-per-op reduction.

\begin{figure}[t]
\centering
\includegraphics[width=\linewidth]{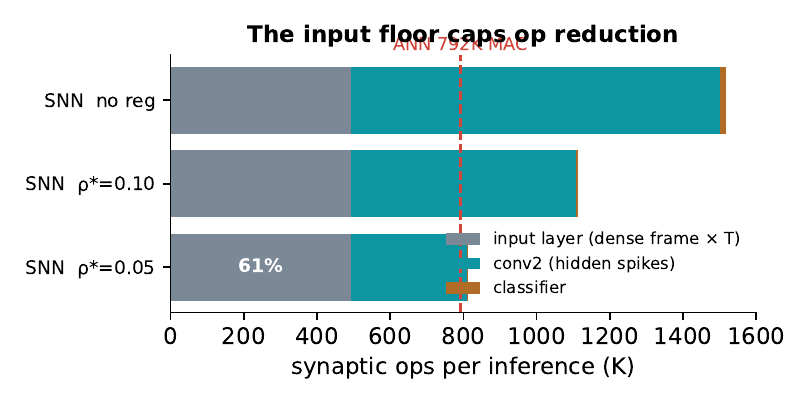}
\caption{Layer-wise synaptic-op budget. As hidden firing is driven down ($\rho^\star$: none$\to$0.10$\to$0.05),
the dense input layer (replayed over $T$) becomes the dominant cost ($61\%$ at $\rho^\star{=}0.05$), holding total
ops at parity with the ANN's MAC count. Native event-driven input would remove this floor.}
\label{fig:floor}
\end{figure}

\subsection{The floor scales with representational load ($N$, $H$, and $C$)}
The bound depends on $\log_2 M/H$, where $M$ is the number of messages the spike code must separate. Two
controlled sweeps test its structure.

\emph{Memory load and width (Fig.~\ref{fig:2d}).} On the copy task ($M{=}S^N$) we sweep memory load $N$ and
hidden width $H$ and read accuracy at a fixed $5\%$ firing budget. Accuracy falls as $N$ grows (the floor rises)
but is bought back by larger $H$ (at $N{=}8$, $H{=}128$ collapses to $48\%$ while $H{=}512$ holds $77\%$) --- both
directions match the $\log_2 M/H$ structure of the bound.

\emph{Difficulty (Fig.~\ref{fig:dissoc}).} We had conjectured the floor tracks \emph{memory} specifically. It does
not: a feed-forward classification with \emph{no} memory but growing class count $C$ also raises the floor
($9.5\%\!\to\!18\%\!\to\!26\%$ for $C{=}10,40,160$). This \emph{refutes} a clean memory-vs-difficulty dissociation
but \emph{confirms the bound's generality}: any source of representational load ($M{=}C$ here) raises $\rho_{\min}$.
The reason sequence modeling is uniquely floored is then sharp: memory makes $M{=}S^N$ grow \emph{exponentially}
in the horizon, whereas typical perception carries a modest $M$. The vision--language gap is memory's exponential
blow-up of $M$, not a categorical memory-only effect --- and ``perception is freely sparsifiable'' holds only at
\emph{low representational load}.

\begin{figure}[t]
\centering
\includegraphics[width=\linewidth]{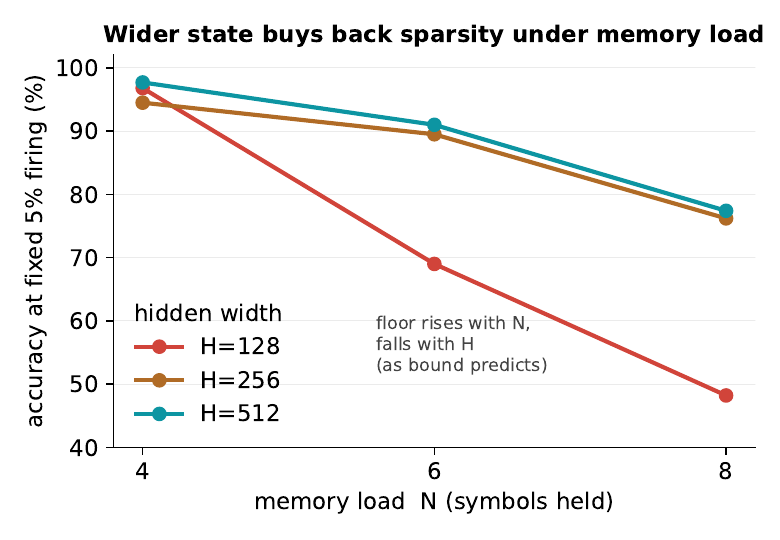}
\caption{Accuracy at a fixed $5\%$ firing budget vs.\ memory load $N$, for three hidden widths $H$. The sparsity
floor rises with $N$ and falls with $H$, as $\rho_{\min}\!\approx\!H_b^{-1}(\log_2 M/H)$ predicts.}
\label{fig:2d}
\end{figure}
\begin{figure}[t]
\centering
\includegraphics[width=\linewidth]{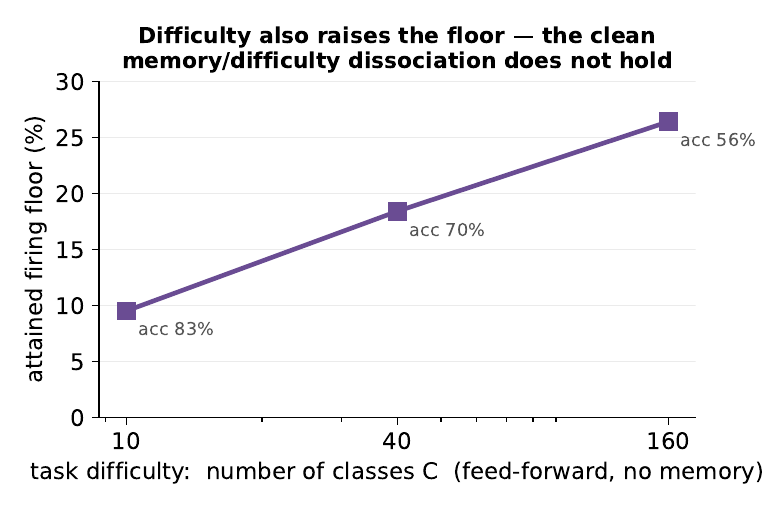}
\caption{A memoryless feed-forward task's firing floor still rises with difficulty $C$: representational load,
not memory alone, sets the floor---consistent with the bound ($M{=}C$).}
\label{fig:dissoc}
\end{figure}

\subsection{Attention sidesteps the ceiling---but pays a memory wall}
Is the ceiling intrinsic to sequence modeling, or to \emph{recurrence}? We build a spiking Transformer (spiking
attention-output and FFN neurons, integrated over $T_s{=}4$ spike-steps) and its matched ANN Transformer on the
same char-level task, and apply the probe (3 seeds).

\begin{center}\small
\begin{tabular}{@{}lcc@{}}
\toprule
config & firing & bpc\\
\midrule
ANN Transformer & --- & $3.05\pm.01$\\
SNN, no reg & $0.27\pm.01$ & $3.55\pm.01$\\
SNN, $\rho^\star{=}0.05$ & $0.050\pm.001$ & $3.57$\\
\textbf{SNN, $\rho^\star{=}0.02$} & $\mathbf{0.020\pm.000}$ & $3.57$\\
\bottomrule
\end{tabular}
\end{center}

\noindent Firing drops $14\times$ (0.27$\to$0.02) with \emph{no} bpc change: the spiking Transformer sparsifies
as freely as feed-forward perception, and shows \emph{none} of the recurrent floor (Fig.~\ref{fig:three}). The
ceiling is thus a property of \emph{recurrent compression}---forcing history through a fixed-width spike state
that must stay active---not of sequence modeling per se. Attention avoids it by keeping all tokens' key--value
entries (random access, no compression), i.e.\ it trades the firing floor for $O(\text{context})$ memory---the
on-chip memory wall that blocks large models on neuromorphic silicon~\cite{northpole}. On sequence tasks,
neuromorphic hardware therefore faces a dichotomy---firing floor (recurrent) or memory wall (attention)---and
escapes neither.

\begin{figure}[t]
\centering
\includegraphics[width=0.92\linewidth]{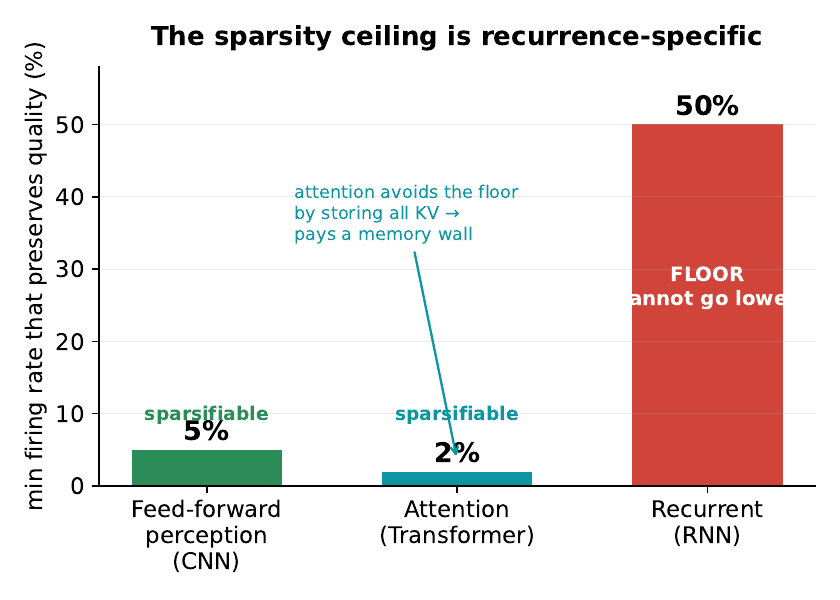}
\caption{Minimum firing that preserves quality, by architecture. Feed-forward perception ($5\%$) and attention
($2\%$) sparsify freely; the recurrent net is floored near $50\%$. Attention buys its sparsity with a KV-cache
memory wall.}
\label{fig:three}
\end{figure}

\section{Implications and related work}
\textbf{Where neuromorphic wins:} sparse, event-driven perception---low attainable firing and, with event
sensors, no input floor. \textbf{Where it does not:} sequence models face a dichotomy. A \emph{recurrent} SNN
compresses history into a fixed spike state and hits the firing floor (no sparsity dividend, only the fixed
AC$<$MAC constant that efficient ANNs also chase). An \emph{attention} SNN sparsifies freely but stores the full
key--value cache---$O(\text{context})$ memory---colliding with the on-chip-memory wall~\cite{northpole} that
already blocks LLM-scale weights on a neuromorphic chip. Either way the cost re-appears; it is not removed.
Consistently, the strongest ``spiking LLM'' results (SpikingBrain~\cite{spikingbrain}, 76B) run on GPU clusters
and spike Transformers~\cite{sdt} largely target vision. High firing is known to erase SNN
efficiency~\cite{sparsefiring,snnlp}; we add the controlled-probe cross-task view, a bound, and the
recurrence-vs-attention cost dichotomy.

\section{Limitations}
(1) Perception uses rate-coded FashionMNIST as an event-stream stand-in; native N-MNIST/DVS and a real
event-input accounting are needed. (2) The spiking Transformer is single-layer, char-level, small-scale; scaling
it (and quantifying the KV-cache memory cost directly) is the key next step. (3) The bound is single-step; a
tightening across $T$ and to task loss (not just distinguishability) is open. (4) Energy is a 45\,nm proxy;
measured Loihi\,2 / SpiNNaker2 energy would convert estimate to measurement. (5) The mechanism sweep uses a
narrow trainable window ($N\le 8$).

\section*{Code and reproducibility}
All scripts, per-run logs, and the energy-model constants are released; every result runs on a single GPU with
no custom hardware. \emph{Repository: \texttt{github.com/zeyuyuyu/sparsity-ceiling} (made public at submission).}

\end{document}